\documentclass[a4paper,fleqn]{cas-sc}

\def\tsc#1{\csdef{#1}{\textsc{\lowercase{#1}}\xspace}}
\tsc{WGM}
\tsc{QE}
\tsc{EP}
\tsc{PMS}
\tsc{BEC}
\tsc{DE}

\usepackage[utf8]{inputenc} 
\usepackage{multirow}
\usepackage{amsmath,amssymb,amsthm}	
\usepackage{dsfont}  
\usepackage{mathrsfs}   
\usepackage{mathtools}
\usepackage{bbm}
\usepackage{ifthen}
\usepackage{graphicx}
\usepackage{subcaption}
\usepackage{booktabs}
\graphicspath{ {./images/} }  
\usepackage{hyperref}	
\usepackage{thm-restate}
\usepackage{color}
\usepackage[dvipsnames]{xcolor}

\usepackage{placeins}

\usepackage{enumitem}
\usepackage[ruled,vlined]{algorithm2e}

\makeatletter
\AtBeginDocument{%
    \@ifundefined{sectionautorefname}
        {}
        {}%
    \@ifundefined{subsectionautorefname}
        {}
        {}%
    \@ifundefined{subsubsectionautorefname}
        {}
        {}%
    \@ifundefined{figureautorefname}
        {}
        {}%
    \@ifundefined{tableautorefname}
        {}
        {}%
    \@ifundefined{equationautorefname}
        {}
        {}%
    \@ifundefined{lemmaautorefname}
        {}
        {}%
    \@ifundefined{theoremautorefname}
        {}
        {}%
    \@ifundefined{propositionautorefname}
        {}
        {}%
    \@ifundefined{definitionautorefname}
        {}
        {}%
    \@ifundefined{remarkautorefname}
        {}
        {}%
    \@ifundefined{assumptionautorefname}
        {}
        {}%
    \@ifundefined{claimautorefname}
        {}
        {}%
    \@ifundefined{algorithmautorefname}
        {}
        {}%
    \@ifundefined{algocfautorefname}
        {}
        {}%
}
\makeatother

\usepackage{nicefrac}

\usepackage{todonotes}

\newtheorem{lemma}{Lemma}
\newtheorem{theorem}{Theorem}
\newtheorem{proposition}{Proposition}

\theoremstyle{remark}

\theoremstyle{remark}
\newtheorem{assumption}{Assumption}

\newcommand{\floor}[1]{\lfloor #1 \rfloor}

\newcommand{\card}[1]{\left \lvert#1 \right \rvert}

\newcommand{\Expectation}[1][]{ 
    \ifthenelse{ \equal{#1}{} }
    {\mathbb{E}}
    {\mathbb{E} \left [ #1 \right ] }
}
\newcommand{\Proba}[1][]{ 
    \ifthenelse{ \equal{#1}{} }
    {\mathbb{P} }
    {\mathbb{P} \left ( #1 \right )}
}
\newcommand{\indicator}[1][]{ 
    \ifthenelse{ \equal{#1}{} }
    {\mathds{1} }
    {\mathds{1} \left \{ #1 \right \}}
}

\newcommand\numberthis{\addtocounter{equation}{1}\tag{\theequation}}

\newcommand{\bvec}{\mathbf{b}}
\newcommand{\betavec}{\boldsymbol{\beta}}

\usepackage{ulem}
\usepackage{eurosym}

\usepackage[authoryear,longnamesfirst]{natbib}

\begin{document}

\let\WriteBookmarks\relax
\setcounter{topnumber}{3}
\setcounter{bottomnumber}{2}
\setcounter{totalnumber}{5}
\renewcommand{\topfraction}{0.9}
\renewcommand{\bottomfraction}{0.8}
\renewcommand{\textfraction}{0.07}
\renewcommand{\floatpagefraction}{0.8}
\shorttitle{Learning to Bid in FCR Markets}
\shortauthors{Potfer et~al.}
\author[1,2]{Marius Potfer}
\cormark[1]
\ead{marius.potfer@ensae.fr}
\credit{Conceptualization, Formal Analysis, Investigation, Methodology, Software, Writing -- Original draft}

\author[2]{Cheng Wan}
\ead{cheng.wan@edf.fr}
\credit{Funding acquisition, Resources, Supervision, Validation, Writing -- Review and editing}

\author[2]{Pierre Gruet}
\ead{pierre.gruet@edf.fr}
\credit{Funding acquisition, Resources, Supervision, Validation, Writing -- Review and editing}

\affiliation[1]{organization={ENSAE, CREST (Fairplay joint team)},
                addressline={5 Av. Le Chatelier},
                postcode={91120},
                city={Palaiseau},
                country={France}}

\affiliation[2]{organization={EDF Lab Paris-Saclay, FiME (Laboratoire de Finance des Marchés de l’Énergie)},
                addressline={7 Bd Gaspard Monge},
                postcode={91120},
                city={Palaiseau},
                country={France}}

\cortext[cor1]{Corresponding author}

\title[mode = title]{Learning to Bid in FCR Markets: A Best-of-Both-Worlds Approach}



\maketitle

\begin{abstract}
Bidding in the European Frequency Containment Reserve (FCR) market is challenging for flexibility providers because competing offers are hidden and bidders observe only partial feedback form the market, such as, clearing price and awarded quantity. For a participant active in a single country, we show that the multi-country FCR clearing problem can be recast as a repeated multi-unit uniform-price auction against an endogenous vector of opposing bids. This reformulation yields an online learning problem and allows us to adapt a Best-of-Both-Worlds combinatorial semi-bandit algorithm implementable from this standard market feedback. The resulting bidder achieves logarithmic pseudo-regret in stochastic environments and $\mathcal{O}(\sqrt{T})$ regret in adversarial ones. Synthetic experiments confirm the expected scaling, and backtests on historical European FCR data show competitive performance in practice: the method performs especially well on stable products, while EXP3-type baselines can be safer under stronger non-stationarity. Overall, the results show that learning-based bidding in FCR markets is theoretically grounded and practically useful when the learning rule matches product-level market stability.
\end{abstract}

\section{Introduction} \label{introduction}

A balance between power supply and demand is necessary to maintain the stability of a power grid.
Depending on the power grid one considers, this balance is ensured through different mechanisms. In most countries, consumers' and producers' schedules are coordinated ahead of time in order to provide this equilibrium. For most European countries, including France, this coordination takes place the day before delivery through a common market mechanism: the wholesale spot market. 

To ensure robustness of the supply/demand equilibrium, one needs to be able to compensate for sudden, unpredictable changes in the behavior of a consumer or a producer (for instance, in case of the outage of a power plant).
This is usually dealt with by leveraging the ability of producers and consumers to change their production or consumption in order to maintain the balance of the power grid. This capacity to modify upward or downward one's power production or consumption is called flexibility.
Transmission System Operators (TSO) make an agreement with flexibility owners to be able to use their flexibility for a fixed price beforehand. These are called \emph{reserve capacity}, and the price and allocation of this reserve capacity are determined by a market mechanism.

To meet operational requirements such as response time, total uptime, and transmission capacities, reserves are categorized into different types. One of these is the Frequency Containment Reserve (FCR), which is the focus of our market analysis. The FCR is activated immediately following a supply or demand imbalance and remains in use until a more sustainable reserve can take over, typically from 15 seconds to a few minutes after the incident.

This work investigates algorithms and methods that enable a market actor to decide the best price to offer its flexibility in the FCR mechanism. To maximize its wealth, each market actor has an incentive to derive relevant participation strategies and to run them consistently over time. This is a challenging problem, as it involves both known factors (such as the market rules and production costs) and unknown ones (notably, the prices offered by competing producers). Moreover, beyond the design of the strategy of participation in an auction, the actor can benefit from the daily repetition of the auction to try to learn the behaviour of other participants.

Our approach builds on the observation that, from the point of view of one participant who is located in a single country, the FCR mechanism can be reduced to a uniform-price auction with a fixed number of units to be procured. While in these auctions, as in most, the best prices to submit depend on other participants' strategies, when the same auction takes place repeatedly, these strategies can be learned through online learning algorithms. 
Because reserve markets are also repeated frequently, this reduction allows us to leverage existing results from the literature on learning in auctions, which we adapt to our specific setting.

\subsection{Literature review}

Online learning for auctions was first studied from the point of view of the auctioneer, notably by \citep{blum2004online}, trying to maximize revenue, or \cite{kanoria2014dynamic}, who focused on learning the reserve pricing. These online procedures were then first explored to be used by participants in these auctions by \cite{weed2016online}. While early work focused on learning single-item auctions, multi-unit auctions were also examined and first mentioned in \cite{feng2018learning}. The uniform price setting, which is of particular interest to our problem, was specifically explored in \cite{branzei2023learning} and then improved by \cite{potfer2024improved}. Further work on these auctions with particular learning objectives or comparing uniform auctions to other auction formats was also developed recently \citep{golrezaei2024bidding,potfer2025comparing}.

Sequential learning has been applied extensively to study energy markets. It has been studied for bidding in energy markets under the adversarial opposing bid hypothesis by \cite{wang2022earning}. A similar approach was presented in \cite{karaca2020no}, where they designed a variation of the EXP3 algorithm for a general model that resembles procurement auctions. Similarly, \cite{abate2024learning} studied how the use of no-regret algorithms by participants in forward electricity markets could influence social welfare. Machine learning-based approaches have also been explored more recently, such as in \cite{YAN2025111320} that leverages deep reinforcement learning or \cite{BEZOLD2025111269}, which focuses on more traditional machine learning techniques such as decision trees or SVMs. 

The online learning algorithms used in the previously described settings benefit from theoretical guarantees. However, these guarantees vary widely depending on the environment (adversarial or stochastic), and one has to choose EXP3 or UCB accordingly \citep{lattimore2020bandit}. In practice, one usually cannot tell beforehand if the environment is stochastic or not. Best Of Both Worlds (BOB) algorithms were introduced in \cite{bubeck2012best} and they adapt automatically to both environments. Recent advances in BOB algorithms allow a simple formulation as an optimization problem \cite{zimmert2021tsallis}. \cite{ito2021hybrid} proposes an algorithm with guarantees not only for adversarial and stochastic environments but also for settings smoothly interpolating the two.

\subsection{Our Contributions}

The main contributions of this work are threefold: (i) modeling of the FCR auction and, in particular, simplification arising from the point of view of a participant, (ii) algorithm design of a bidding algorithm and theoretical guarantees for the performance, (iii) numerical simulation and backtesting on real data. 

This work provides a precise description of the FCR market rules and of the details of the clearing mechanism available on the website~\cite{FCR-clearing-result}. We give a formal modeling of this mechanism with several countries interconnected by a given capacity and a description of the price fixing and allocation rules. Focusing on a utility-maximizing flexibility provider's bidding strategy, we show that the problem he faces can be reduced to bidding in a uniform price auction, under weak assumptions. In particular, this remains valid in practice when only accepted bids are revealed.  
This reduction is a key contribution of this work. Its benefits are twofold: first, it allows us to disregard irrelevant variables such as demand in foreign countries, and second, it allows us to leverage on the existing literature on bidding in uniform price auctions. 

Building on the aforementioned reduction, we adapt existing online learning algorithms for learning to bid in auctions to the problem faced by a flexibility provider. Specifically, we focus on BOB algorithms which provide optimal theoretical guarantees on the performance without prior knowledge of the behaviour of other participants. We particularly leverage the ones developed for semi-bandit feedback \cite{zimmert2019beating}: their structure closely matches that of the auction studied in this paper, and in our reformulation, the corresponding coordinate-level signals can be reconstructed from the market-level bandit observations (allocation and clearing price). We also provide an efficient method to solve the optimization problem at each step of the BOB algorithm, which allows the algorithm to run in a reasonable time. 

To illustrate our results, we provide both synthetic market simulations and simulations based on actual reserve market data. The synthetic simulations showcase the guarantees provided by our learning procedure. The simulations based on actual data show that the approach can be used in practice to learn bid prices in the FCR market, and highlight that algorithm choice should depend on product characteristics: BOB performs well on more stable products, while EXP3-type methods can be preferable when non-stationarity is stronger.

The remainder of the paper is organized as follows. \autoref{sec : fcr market} formalizes the FCR clearing mechanism and states the reduction to a uniform-price auction from the perspective of one participant. We then introduce, in \autoref{sec:online learning section} , the repeated-auction learning setting (the feedback models and regret benchmarks) as well as our algorithm and its guarantees. The numerical study, which evaluates the method on both synthetic and historical FCR data is presented in \autoref{sec:numerics}. We finally conclude and summarize our findings in \autoref{conclusion}; the appendix (\autoref{app}) gathers the full proofs and technical constructions.

\section{FCR Market Modelling and reduction to Auction}\label{sec : fcr market}

In this paper, we focus on the European Frequency Containment Reserve (FCR) cooperation, operated under the ENTSO-E framework and national TSOs' joint procurement rules.
The FCR is a symmetric reserve, meaning it can be activated in both upward and downward directions.  As a result, we consider symmetric flexibility, without distinguishing between upward and downward services.
Furthermore, since both consumers and producers can provide flexibility, we will refer to both types of actors as flexibility providers or participants. The FCR divides electricity into units of 1 MW, which we will call units in the following.

The general rules of this mechanism are drafted in the official report \citep{proposal-FCR-EU}, and daily clearing data is available in \cite{FCR-clearing-result}.
The model presented below omits some operational refinements but preserves the core dynamics that shape the price formation. In particular, it accounts for the multi-country nature of the mechanism and the restricted interconnection capacities between national grids. These two elements together create discontinuities in the price-setting process.

Let $\mathcal{S}$ be the set of countries involved in the mechanism. For each $s \in \mathcal{S}$, denote $\mathcal{J}_s$ the set of flexibility providers within that country, and $\mathcal{J}:= \cup_{s \in \mathcal{S}} \mathcal{J}_s$ the set of all flexibility providers.

\begin{assumption}
We assume in this paper that each flexibility provider only participates in one country. Mathematically, it means that the sets $(\mathcal{J}_s)_{s \in \mathcal{S}}$ are disjoint.
\end{assumption}

The FCR market proceeds as follows:
\begin{enumerate}
\item For each country $s \in \mathcal{S}$, the demand for reserve $D_s$ and the maximum available transmission capacity $T_s$ (the maximum amount of electricity that can be transferred in or out of country $s$) dedicated to reserve are made public.
\item Each flexibility provider $j \in \mathcal{J}$ submits a set of bids denoted $\bvec_j$, reflecting how many units they are willing to provide and at which prices.
\item The market clears, determining one price $p_s$ for each country $s \in \mathcal{S}$ (which applies to all providers from that country), and an allocation $x_j$ for each flexibility provider $j \in \mathcal{J}$.
\item Each flexibility provider $j \in \mathcal{J}$ located in country $s$, is paid the price $p_s$  for each unit of reserve it is allocated, i.e., it receives $x_j p_s$ euros.
\end{enumerate}

We define the maximum quantity of reserve that can be procured in country $s$ by $D_{max,s}:= D_s + T_s$.

The following paragraph provides details on the \emph{bidding} procedure and the following \autoref{sec:clearing} expands on the market clearing mechanism.

\paragraph{Bidding} \label{sec : bidding}
In practice, flexibility providers submit bids in the form of a set of prices and corresponding quantities, representing how many units they are willing to sell above each price. For ease of notation, we assume that the bid prices (which in practice have upper and lower bounds) are rescaled between $0$ and $1$.
We also assume that flexibility providers submit one price for each unit of reserve they are willing to sell.

\begin{assumption}
For every auction and every country $s \in \mathcal{S}$, the clearing price satisfies $p_s < 1$.
\end{assumption}

Under this assumption, we can complete each offer up to $D_{\max,s}$ by inserting the maximal price $1$ without changing outcomes: these completed bids are never accepted and therefore do not affect prices or allocations. Empirically, this assumption is consistent with the historical FCR data considered in this paper.\newline
Since bids are made in euros, the smallest price increment is one cent. The rescaled bids therefore belong to a corresponding uniform discretization of $[0,1]$, denoted by $\mathcal{B}$. Hence, to participate in the market, each flexibility provider emits a bid vector $\bvec_j \in \mathcal{B}^{D_{max,s}}$, whose coordinates are assumed to be ordered in non-decreasing order.

\subsection{Clearing FCR market} \label{sec:clearing}

The clearing of the market is described in the following section. It aims at minimizing the total cost of procuring the required reserve across all countries, while respecting national demand constraints, transfer limits, and the bids of all flexibility providers. It is formulated as a constrained minimization problem which determines both the allocation $\left( x_j \right)_{j \in \mathcal{J}}$ (i.e., how many units of reserve each provider sells) and the per-unit price paid in each country $\left( p_s \right)_{s \in \mathcal{S}}$:
\begin{equation}
\left( (x_j)_{j \in \mathcal{J}}, (p_s)_{s \in \mathcal{S}} \right)
= \underset{ \substack{ \mathbf{p} \in \mathcal{P}(\bvec) \\ (x_j)_{j \in \mathcal{J}} \in \mathcal{A}(\bvec, \mathbf{p}) }}{\arg\min}
\sum_{s \in \mathcal{S}} \sum_{j \in \mathcal{J}_s} x_j \, p_s ,
\label{eq:fcr_minimization}
\end{equation}
where $\mathcal{P}(\bvec)$ denotes the set of admissible price vectors given the submitted bids $\bvec$, and $\mathcal{A}(\bvec, \mathbf{p})$ the corresponding set of admissible allocations. We formally define these constraint sets below. 

\paragraph{Bids and Allocations.}
Each provider $j \in \mathcal{J}_s$ in country $s$ submits a set of bids $\bvec_j = (b_{j,l})_{l \in [D_{max,s}]} \in \mathcal{B}^{D_{max,s}}$. For a given price $p_s$, the maximum allocation a provider can receive is the number of units it offered below that price:
\[
\bar{x}_j(\bvec_j, p_s) := \sum_l \indicator\{ b_{j,l} \le p_s \}.
\]
Similarly, we can define the minimum allocation a flexibility provider can receive as the number of bids strictly below the price \[
\underline{x}_j(\bvec_j, p_s) := \sum_l \indicator\{ b_{j,l} < p_s \}.
\]
We denote by $x_j(\bvec_j,p_s)$ the actual allocation chosen by the clearing mechanism; it is an integer satisfying $\underline{x}_j(\bvec_j, p_s) \leq x_j(\bvec_j,p_s) \le \bar{x}_j(\bvec_j, p_s)$.

\paragraph{Transfers.}
For each pair of countries $(s, s') \in \mathcal{S}^2$, let $t_{s,s'} \in \mathbb{R}^+$ represent the transfer from $s$ to $s'$, and let $\mathcal{T}$ denote the set of \emph{valid transfers}, i.e.,
\begin{enumerate}[label=\roman*]
\item $\forall s \in \mathcal{S}, \sum_{s' \ne s} t_{s,s'} \le T_s$ (exports bounded by capacity),
\item $\forall s \in \mathcal{S}, \sum_{s' \ne s} t_{s',s} \le T_s$ (imports bounded by capacity),
\item $\forall s,s' \in \mathcal{S}^2, \ t_{s,s'} \cdot t_{s',s} = 0$ (no simultaneous import and export with a specific country).
\end{enumerate}
We say that the transmission capacity of country $s$ is \emph{saturated} if either its total imports or exports reach $T_s$.

\paragraph{Admissible Prices and Allocations.}
A vector of prices $\mathbf{p} = (p_s)_{s \in \mathcal{S}}$ is \emph{admissible} given the bids $\bvec$ if there exist valid transfers $\mathbf{t} = (t_{s,s'})_{s,s' \in \mathcal{S}^2} \in \mathcal{T}$ such that:
\begin{align}
& \forall s \in \mathcal{S}, \
\sum_{j \in \mathcal{J}_s} \bar{x}_j(\bvec_j, p_s)
\ge D_s + \sum_{s' \in \mathcal{S}} (t_{s,s'} - t_{s',s}),
\label{eq:clearing_demand_transfer} \\
& \exists p \in \mathcal{B}, \forall s \in \mathcal{S} \begin{cases}
\text{if} \left| \sum_{s'} (t_{s,s'} - t_{s',s}) \right| < T_{s} & \text{then } p_s =p \\
\text{if}  \sum_{s'} (t_{s,s'} - t_{s',s}) = T_s & \text{then } p_s \leq p \\
\text{if}  \sum_{s'} (t_{s,s'} - t_{s',s}) = - T_s & \text{then } p_s \geq p 
\end{cases}
\label{eq:clearing_prices_international}
\end{align}

Condition \eqref{eq:clearing_demand_transfer} on admissible prices ensures each country can meet its (demand + net exports). Condition \eqref{eq:clearing_prices_international} ensures that prices behave as in the true FCR market. It enforces equal prices across countries connected by non-saturated transfers and, respectively, higher/lower prices when import/export capacities are saturated.

Given admissible prices $\mathbf{p}$, an admissible allocation $\left( x_j \right)_{j \in \mathcal{J}}$ satisfies:
\begin{align}
& \forall s \in \mathcal{S}, \forall j \in \mathcal{J}_s, \underline{x}_j(\bvec_j, p_s) \leq  x_j \le \bar{x}_j(\bvec_j, p_s), \label{eq : encadrement allox} \\
& \forall s \in \mathcal{S}, \ \sum_{j \in \mathcal{J}_s} x_j
= D_s + \sum_{s'} (t_{s,s'} - t_{s',s}).
\label{eq:clearing_allocation_balance}
\end{align}

\subsection{Utility maximizing flexibility provider } \label{sec : simplification when focusing on bidding}

We focus on the perspective of a flexibility provider participating in the FCR market from a single country, that we will also call the \textit{principal}. It faces the problem of selecting bid values to maximize its expected utility. We fix this provider as $j \in \mathcal{J}_s$, in country $s$, for the rest of the paper. We adopt the following quasi-linear model for its utility: let $\left \{ c_1,\hdots,c_{D_{max,s}} \right \} \in [0,1]^{D_{max,s}}$ be the marginal costs of producing each unit of flexibility sold. We assume these marginal costs are non-decreasing. The utility of the principal after the auction writes as follows : \begin{align} \label{eq : utility flex producer}
u_j(\bvec) := \sum_{i=1}^{x_j(\bvec)} \left ( p_s(\bvec) - c_i \right )
\end{align}

From the point of view of the principal, the only relevant outcomes of the market clearing are its own allocation $x_j(\bvec)$ and the local clearing price $p_s(\bvec)$ (this is clear from the utility formula \eqref{eq : utility flex producer}). This results in the following: everything happens as if it were participating in a uniform price auction (introduced below and fully described in \autoref{app : lemma equivalence}), which is simpler than the clearing mechanism described in \autoref{sec:clearing}. The following \autoref{lemma : equivalence market-auction} provides a formal statement of this reduction.

\subsection{Reduction to a Uniform Auction}

The equivalence described below provides a direct structural justification for applying existing online learning techniques, developed for uniform price auctions, to the real FCR market. Before stating our result, let us briefly introduce the mechanism that is the focus of this equivalence. 

\paragraph{Uniform Price auction} 
A uniform-price auction is a mechanism that selects the \(D_{max,s}\) lowest bids from a set of submitted bids. It accepts the corresponding units and pays every accepted unit the same clearing price, equal to the last (highest) accepted bid. We denote by \((p^U(.),x^U(.))\) the resulting clearing price and accepted quantity (called allocation). Full details are provided in \autoref{sec:online learning}.

We denote $\bvec_{-j,s}$ and $\bvec_{-s}$ respectively the bids of the other producers in country $s$ and the bids of producers in all other countries $s' \neq s$. Then the following holds: 

\begin{restatable}{theorem}{mymaintheorem} \label{lemma : equivalence market-auction}
There exists a constructively computable set of bids $\betavec (  \bvec_{-j,s},\bvec_{-s} ) $, such that for any bids $\bvec_j$ of the principal: \begin{itemize}
\item The clearing price of FCR mechanism \( p_s(\bvec_j,\bvec_{-j})\) equals the price $p^\star(\bvec_j,\betavec)$ of the uniform price auction procuring $D_{max,s}$ units. 
\item The allocation of the FCR market $x_j(\bvec_j,\bvec_{-j})$ and of the uniform auction $x(\bvec_j,\betavec)$ are equal.
\end{itemize} 
Therefore, from the perspective of a utility-maximizing flexibility provider, as described in \autoref{sec : simplification when focusing on bidding}, both mechanisms are equivalent.
\end{restatable}

The full proof is available in Appendix \autoref{app : lemma equivalence}. We provide below a proof sketch to provide insight into both how $\betavec$ can be computed, and where the equivalence between the two mechanisms arises.

\begin{proof}[Proof sketch]
Let $K := D_{\max,s}$. The first key structural input is the monotonicity lemma from Appendix~\autoref{lem:monotonicity_clearing_bids}: if bidder $j$ lowers only bids that were already accepted, or raises only bids that were already rejected, then $j$'s allocation does not change. \newline
This implies that each unit index $m\in[K]$ has a threshold: keeping all bids from other flexibility providers fixed, the principal sells unit $m$ if and only if $b_{j,m}$ is below that threshold. By denoting each $\betavec=(\beta_1,\dots,\beta_K)$ the non-decreasing vector built from these thresholds, we can write the allocation as follows :
\[
x_j(\bvec_j,\bvec_{-j}) = \sum_{m=1}^{K} \indicator\{b_{j,m}\le \beta_{K-m+1}\}.
\]
This is exactly the acceptance rule of a $K$-unit uniform-price auction against opposing bids $\betavec$.

Having established the equivalence of the allocation in the FCR with that of a $K$-unit uniform-price auction against opposing bids $\betavec$, it only remains to show price equivalence (as per  \eqref{eq : utility flex producer}) for both utilities to be equal. 

We prove that the FCR price must be equal to the price in the equivalent $K$-unit uniform-price auction against opposing bids $\betavec$ by leveraging the previously shown allocation equality, \eqref{eq : encadrement allox}, and the definitions of $\bar{x}$ and $\underline{x}$.
The Appendix \autoref{app : lemma equivalence} provides the complete formal arguments that make this sketch rigorous.
\end{proof}

\section{Online Learning}\label{sec:online learning section}

We focus on the problem faced by the principal: the choice of the best bid $\bvec$ in order to maximize its utility. We are interested in leveraging the repeated aspect of the FCR auction to accumulate knowledge about \textit{the behaviour} of opposing bids. We formally describe below the uniform auction, the repeated setting as well as the information available to the principal, allowing them to \emph{accumulate knowledge}.

\subsection{Setting and Notations}\label{sec:online learning in auction}

\autoref{lemma : equivalence market-auction} demonstrates that, when considering the problem faced by the principal, the FCR auction can be viewed as an instance of an uniform auction, where $D_{max,s}$ units are to be bought and in which other participants' bids are $\boldsymbol{\beta}$. Therefore, from this section onward, we work in the reduced uniform-auction representation, which is sufficient because of \autoref{lemma : equivalence market-auction}. Since there is no ambiguity remaining, we write $p(\bvec,\betavec)$, $x(\bvec,\betavec)$, $u(\bvec,\betavec)$ for outcomes instead of $p^{\mathrm U}(\bvec,\betavec)$, $x^{\mathrm U}(\bvec,\betavec)$, $u^{\mathrm U}(\bvec,\betavec)$. We detail below how this auction proceeds:

\paragraph{Uniform Price auction}  \label{sec:online learning} We describe below in detail the uniform price auction mentioned above, where $D_{max,s}$ units are to be bought, from the point of view of the principal, whose cost of producing the $\text{k}^\text{th}$-unit of electricity is denoted by $c_{k}\in [0,1]$. In this auction, the principal faces opposing bids $\betavec$ (which, by \autoref{lemma : equivalence market-auction} can be computed from actual opposing bids in FCR markets). The auction proceeds as follows:

\begin{enumerate}
\item The principal submits its bids $\bvec :=(b_{k})_{k\in[K]} \in B$, where $K=D_{max,s}$ and $B= \left \{(b_l)_{l \in [K]} \in \mathcal{B}^{D_{max,s}} \right .$, such that $ \left. 0 \leq b_{1}\leq b_{2} \leq \hdots \leq b_{K} \leq 1 \right \}$.
\item The opposing bids $\betavec = (\beta_k)_{k \in [K]} \in B$ (in non-decreasing order)\footnote{We take $\betavec \in B$, this is without loss of generality : if $\boldsymbol{\beta}$ has more than $D_{max,s}$ elements then, when only keeping the $D_{max,s}$ the auction results remain the same regardless of $\bvec_j$  } are submitted to the auction.
\item The per-unit price is set as the $D_{max,s}^\text{th}$ lowest bid from $(\bvec_j,\betavec)$, denoted by $p \left(\bvec_j, \betavec \right)$.
\item The principal sells to the auctioneer every unit of flexibility they proposed below the price $p \left( \bvec_j,\betavec \right)$ and receives this price in exchange. Their allocation $x\in[K]$ of production is therefore as follows\footnote{This allocation breaks ties in favor of the flexibility provider for simplicity; our results still follow through if tie-breaking is random or always against the principal.}:
\begin{align*}
x\left (\bvec,\betavec \right ) & := 
\card{ \left \{k \in [K] \text{ s.t. } b_{k} \leq p \left ( \bvec,\betavec \right ) \right \} } .\numberthis \label{def : allocation elec market}
\end{align*}
\end{enumerate}

This setup gives rise to the quasi-linear utility $u(\bvec,\betavec) = \sum_{l=1}^{x(\bvec,\betavec)} \left [ p(\bvec,\betavec) - c_l\right]$.

Let $T$ be the time horizon of the repeated auction. Sequentially, at each time $t \in [T]$, the auction proceeds as described in \autoref{sec:online learning}, or equivalently as described above. We extend the notations introduced above to describe the auction with a superscript $t$ (i.e., the principal bids $\bvec^t$, opposing bids $\betavec^t$, etc.). The following convention allows us to completely describe the online learning problem faced by the principal.

\paragraph{Feedback} For learning, specifying the information available after each round, usually referred to as \emph{feedback}, is pivotal. We describe below the two types of feedback we consider.  \begin{itemize}
\item Full information feedback: All the information about the auction is revealed to the principal, formally, the principal observes $\betavec^t$.
\item Bandit Feedback: The principal only observes its allocation $x(\bvec^t,\betavec^t)$ and the price $p(\bvec^t,\betavec^t)$.
\end{itemize}

In the original auction representation, it is common to only receive Bandit Feedback. As such, the algorithm we propose below (\autoref{algo : bob_bidder}) works with this minimal feedback. Naturally, if the principal receives richer feedback, \autoref{algo : bob_bidder} remains valid and can be implemented by reconstructing the weaker bandit feedback.  

\paragraph{Opposing bids} We consider two possibilities regarding the process generating the opposing bids $\left ( \betavec^t \right )_{t\in [T]}$. Either the opposing bids are adversarial, or they are stochastic. We describe below these two settings as well as the two notions of regret, a common benchmark in online learning, which correspond to the settings. \begin{itemize}
\item The opposing bids are stochastic. This setting is defined by $\mathcal{D}$, a distribution from which, for each timestep $t \in [T]$, the opposing bids $\betavec^t$ are sampled.
\item The opposing bids are adversarial. In this case, the opposing bids $(\betavec^t)_{t \in [T]}$ can be any fixed sequence.
\end{itemize}

Notice that the stochastic setting is a special case of the adversarial one. The purpose of having these two separate settings is to be able to provide stronger results in the stochastic setting.

\paragraph{Regret} It is common in online learning literature to quantify the quality of a learning algorithm by its \textit{regret}: the difference between the cumulative utility received and the one generated by the best bid in hindsight. We define it below: 
\begin{equation}
R_T = \sup_{\bvec \in B}{\mathbb{E}}  \left [  \sum_{t=1}^T u(\bvec,\betavec^t) -  \sum_{t=1}^T u(\bvec^t,\betavec^t) \right ]
\end{equation}

In the stochastic setting, the pseudo-regret is a similar benchmark which takes into account the random nature of the rewards: \begin{equation}
\tilde{R}_T = T \sup_{\bvec \in B}  \underset{ \betavec \sim \mathcal{D}}{\mathbb{E}}  \left [  u(\bvec, \betavec) \right ] - {\mathbb{E}}  \left [  \sum_{t=1}^T  u(\bvec^t, \betavec^t) \right ]
\end{equation}

\subsection{Proposed algorithm}

We describe in this section the algorithm we propose to solve the regret minimization problem introduced above, by first explaining how the bidding problem faced by the principal maps to a combinatorial semi-bandit model, then stating the chosen algorithm and its regret guarantees, and finally discussing how to compute its key minimization step efficiently. Our mapping to combinatorial semi-bandits builds upon insights from both \cite{branzei2023learning} and \cite{potfer2024improved}, which study the same type of multi-unit auction \autoref{sec:online learning in auction} in a repeated setting. We combine this structural insight with the online learning algorithm from \cite{zimmert2019beating}, which adapts to both stochastic and adversarial instances.

Using the results from Lemma 1 and Lemma 3 in \cite{potfer2024improved} allows us to transform our bidding space $\mathcal{B}$ into $\mathcal{H} \subseteq \{0,1\}^N$ (with a one-to-one correspondence between $\mathcal{B}$ and $\mathcal{H}$). We can also rewrite the utility as follows:\begin{equation}
u(\bvec,\betavec) = \mathbf{h} (\bvec) \mathbf{W}(\betavec)
\end{equation}

We provide a quick summary of how $\mathbf{h}(\cdot)$ and $\mathbf{W}(\betavec)$ are defined in the appendix, \autoref{app : rewriting utility}. 

This rewriting of the action space and utility matches the combinatorial semi-bandits framework presented in \cite{zimmert2019beating}, allowing us to leverage their Best-of-Both-Worlds algorithm's results. The algorithm is detailed in \autoref{algo : bob_bidder} and uses the following regularizer, parameterized by $\gamma \in (0,1]$: \begin{equation} \label{def : regularizer} \Psi_{\gamma} (x) := \sum_{i=1}^N - \sqrt{x_i} + \gamma \left (  1- x_i \right ) \log \left (  1- x_i \right ) \end{equation}

The update in \autoref{algo : bob_bidder} is written in semi-bandit form in the transformed action space $\mathcal{H}$. The principal observes market-level bandit feedback only, i.e., the clearing price and obtained allocation $(p_t,x^t)$. The reconstruction claim is made explicit in \autoref{prop:bandit_to_semibandit_reconstruction}: once $\bvec^t$ is known, active-coordinate outcomes $y_{t,i}=h^t_iW_i(\betavec^t)$ are functions of $(\bvec^t,p_t,x^t)$ for coordinates with $h_i^t=1$. Hence the estimator is implementable from bandit feedback in the original auction representation; richer feedback (winning bids or full $\betavec^t$) is optional.

\begin{algorithm}
\caption{\textsc{BOB Algorithm for Semi-Bandit}}
\label{algo : bob_bidder}
\KwIn{$0 < \gamma \leq 1$, sampling scheme $P$}
\textbf{Initialize:} $\hat{W}_0 = (0, \ldots, 0)$, $\eta_t = 1/\sqrt{t}$\;
\For{$t = 1, 2, \ldots$}{
    Compute 
    \begin{equation} \label{eq : algorithm compute min}
        x_t = \arg\min_{\mathbf{h} \in \text{Conv}(\mathcal{H})} \left\langle \mathbf{ h }, \hat{W}_{t-1} \right\rangle + \eta_t^{-1} \Psi(h)
    \end{equation}
    where $\Psi(\cdot)$ is defined in \autoref{def : regularizer}\;
    
    Sample $\mathbf{h}_j^t \sim P(x_t)$\;
    
    Observe market-level bandit feedback $(p_t, x_j^t)$\;

    Reconstruct induced semi-bandit outcomes $y_{t,i}$ on active coordinates (via the
    $B \leftrightarrow \mathcal{H}$ mapping and \autoref{app : rewriting utility}), with
    \begin{equation*}
        y_{t,i} = h^t_{i} W_i(\betavec^t) \quad \text{when } h^t_{i}=1
    \end{equation*}
    
    Construct estimator $\hat{w}_t$, $\forall i \in [N]$: 
    \begin{equation*}
        \hat{w}_{t,i} = \frac{(y_{t,i}+K)\mathbb{I}_{t}(i)}{x_{t,i}} - K,\qquad \mathbb{I}_{t}(i)=h^t_{i}
    \end{equation*}
    
    Update $\hat{W}_t = \hat{W}_{t-1} + \hat{w}_t$\;
}
\end{algorithm}

We restate below a simplified version of Theorem 2 from \cite{zimmert2019beating}, whose guarantees extend to our setting: 

\begin{theorem}
The pseudo regret incurred by \autoref{algo : bob_bidder} is upper bounded, in the stochastic opposing bid case, by \begin{equation}
\tilde{R}_T = \mathcal{O} \left ( \log T \right )
\end{equation}
and in the adversarial opposing bid case by \begin{equation}
R_T = \mathcal{O} \left ( \sqrt{T} \right )
\end{equation}
where $\mathcal{O}$ hides problem-dependent constants.
\end{theorem}

\paragraph{Computing the minimum efficiently}

The previous theorem ensures the good performance of \autoref{algo : bob_bidder} according to our benchmark. Since our goal is to produce a working algorithm with a real application in mind, we now discuss the implementability of the minimization step in \autoref{eq : algorithm compute min}.

In our setting, this step amounts to solving a constrained convex optimization problem on $\mathrm{Conv}(\mathcal{H})$. We use a conditional gradient (Frank--Wolfe) procedure, because the feasible set is a polytope with a simple combinatorial structure and the corresponding linear minimization oracle is tractable: each oracle call reduces to selecting an extreme point in $\mathcal{H}$, which can be formulated as a maximum-weight path problem in a directed acyclic graph.

Put simply, the optimization step is tractable in practice: with our current implementation, in the FCR setting we can run the method for all required values up to $K=500$ at the target precision in less than one minute per iteration.

                 \section{Numerical Study}\label{sec:numerics}

Finally, we apply the algorithm we designed to bid in FCR markets. We conduct two types of numerical study: one on synthetic data and one on real data from the FCR market. The complete code used to conduct this analysis is made available as supplementary material attached to this work, and FCR data is accessible at \cite{FCR-clearing-result}.
                                                                                                                   
To evaluate our algorithm fairly, we compare it to an EXP3-based algorithm developed in \cite{potfer2024improved}. This is a common no-regret baseline and is used in similar settings by \cite{abate2024learning}. 

\subsection{Synthetic data}

In order to fully assess the behaviour of the algorithm we developed, we conduct synthetic tests to compare \autoref{algo : bob_bidder} with the EXP3 baseline.  As mentioned above, the algorithm we developed benefits from asymptotic theoretical performance guarantees (and as such may hide large constants in $\mathcal{O}$ notation). These tests complement asymptotic results by examining empirical short-term behaviour and by serving as implementation sanity checks. 
In these synthetic experiments, we use $K=4$ units and a bid discretization step of $0.1$.  These two choices keep the action space relatively small and computation time low.

\begin{figure}[!htbp]
    \centering
    \begin{subfigure}{0.48\linewidth}
        \centering
        \includegraphics[width=\linewidth]{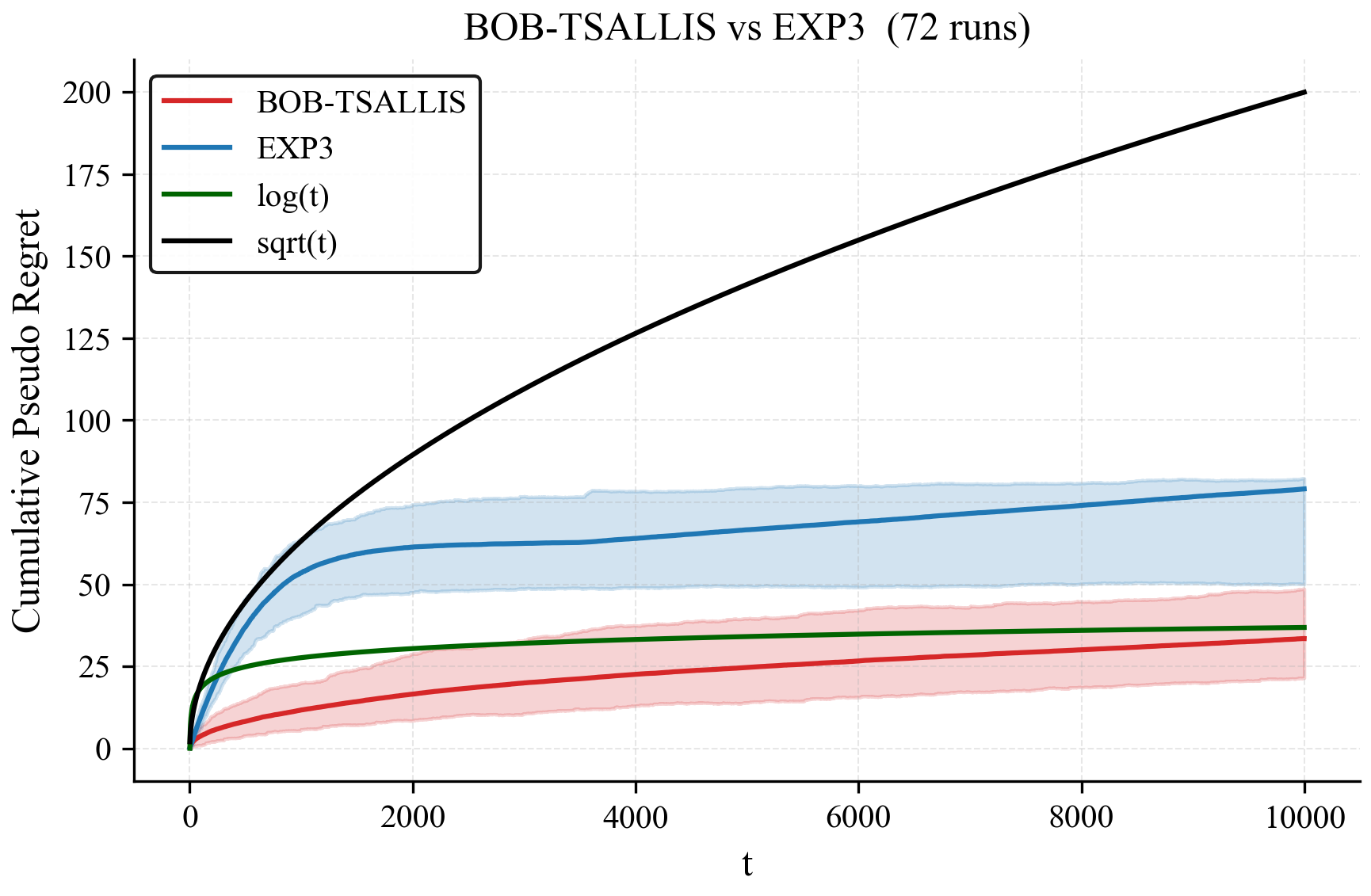}
        \caption{Bandit setting}
        \label{fig:synthetic_pseudo_regret}
    \end{subfigure}
    \hfill
    \begin{subfigure}{0.48\linewidth}
        \centering
        \includegraphics[width=\linewidth]{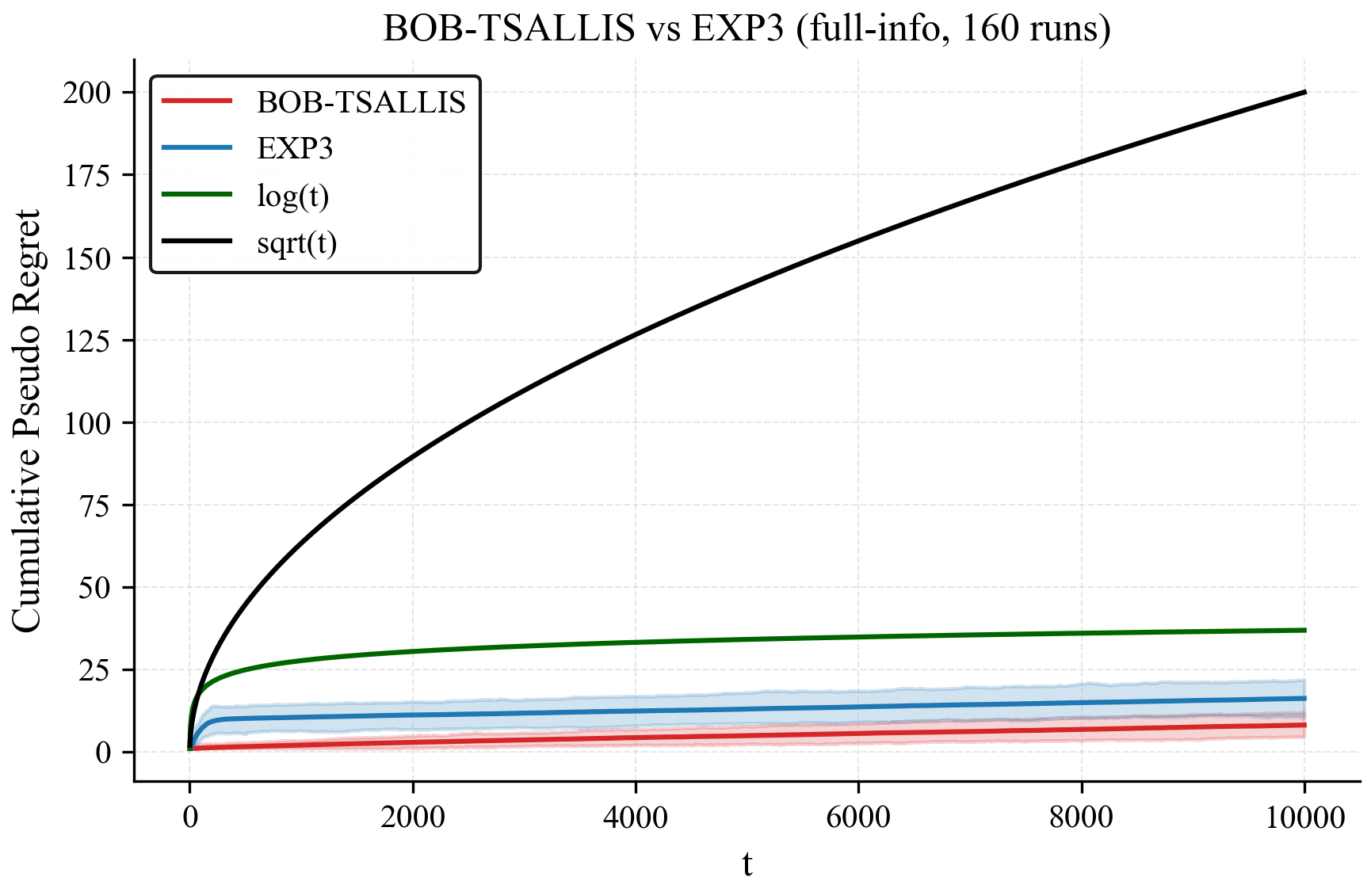}
        \caption{Full information setting}
        \label{fig:output_full_info}
    \end{subfigure}
    \caption{Pseudo-regret for BOB and EXP3 over 10000 rounds}
\end{figure}

\autoref{fig:synthetic_pseudo_regret} reports pseudo-regret trajectories in the synthetic setting for BOB and EXP3 over $10{,}000$ rounds, both receiving bandit feedback and facing uniformly distributed opposing bids. The figure also shows the reference curves $\sqrt{t}$ and $\log(t)$ to compare empirical growth with theoretical rates. The curves are averages over repeated runs, and the colored areas represent $95\%$ intervals (from the $2.5\%$ to the $97.5\%$ quantiles). In this experiment, both algorithms remain below their corresponding reference scales, which is consistent with the expected qualitative behaviour.

We also provide synthetic result of pseudo regret of both learning, under the richer full information feedback (when all opposing bids $\betavec^t$ are revealed) in \autoref{fig:output_full_info}, with the same experimental setting. As expected under this more informative setting, both algorithms drastically improve the incurred pseudo-regret, while our improved \autoref{algo : bob_bidder} still outperforms EXP3 algorithms. Overall, the synthetic study validates expected scaling and implementation behaviour, while the next subsection tests performance under realistic market conditions.

\subsection{FCR market Data}

We now focus on using historical data in order to quantify how \autoref{algo : bob_bidder} performs. 

 \paragraph{Experimental setup.}
We study repeated bidding in the European FCR capacity market from the point of view of one French flexibility provider. The analysis is based on public clearing data (offers, demand, and market outcomes) from July 1, 2020 to May 31, 2024. As in the market design itself, we treat the six 4-hour products (\texttt{NEGPOS\_00\_04}, \texttt{04\_08}, \texttt{08\_12}, \texttt{12\_16}, \texttt{16\_20}, \texttt{20\_24}) as distinct auction environments. This is important in practice, because observed price levels and variability differ substantially across products, and pooling them would mix structural intraday effects with the learning dynamics we want to measure.

Full implementation details are provided in the released code and reproducibility material, available as supplementary material, attached to this work. 

The learning protocol is online and therefore fully chronological. At each auction date, the learner submits a bid, observes market-level bandit feedback (clearing price and awarded quantity), reconstructs the coordinate-level outcomes required by \autoref{algo : bob_bidder} through the $B \leftrightarrow \mathcal{H}$ reformulation, and updates before the next auction. This matches the feedback model used in our theoretical guarantees. We compare BOB and EXP3 on exactly the same product-specific auction streams, with the same marginal costs of production, action discretization, and time horizon. For each product and each algorithm, we perform 102 Monte Carlo runs (with non-fixed seeds), and define regret against the best fixed bid in hindsight on the same evaluation window.

\begin{figure}[!htbp]
    \centering
    \begin{subfigure}
        {0.48\linewidth}
        \centering
        \includegraphics[width=\linewidth]{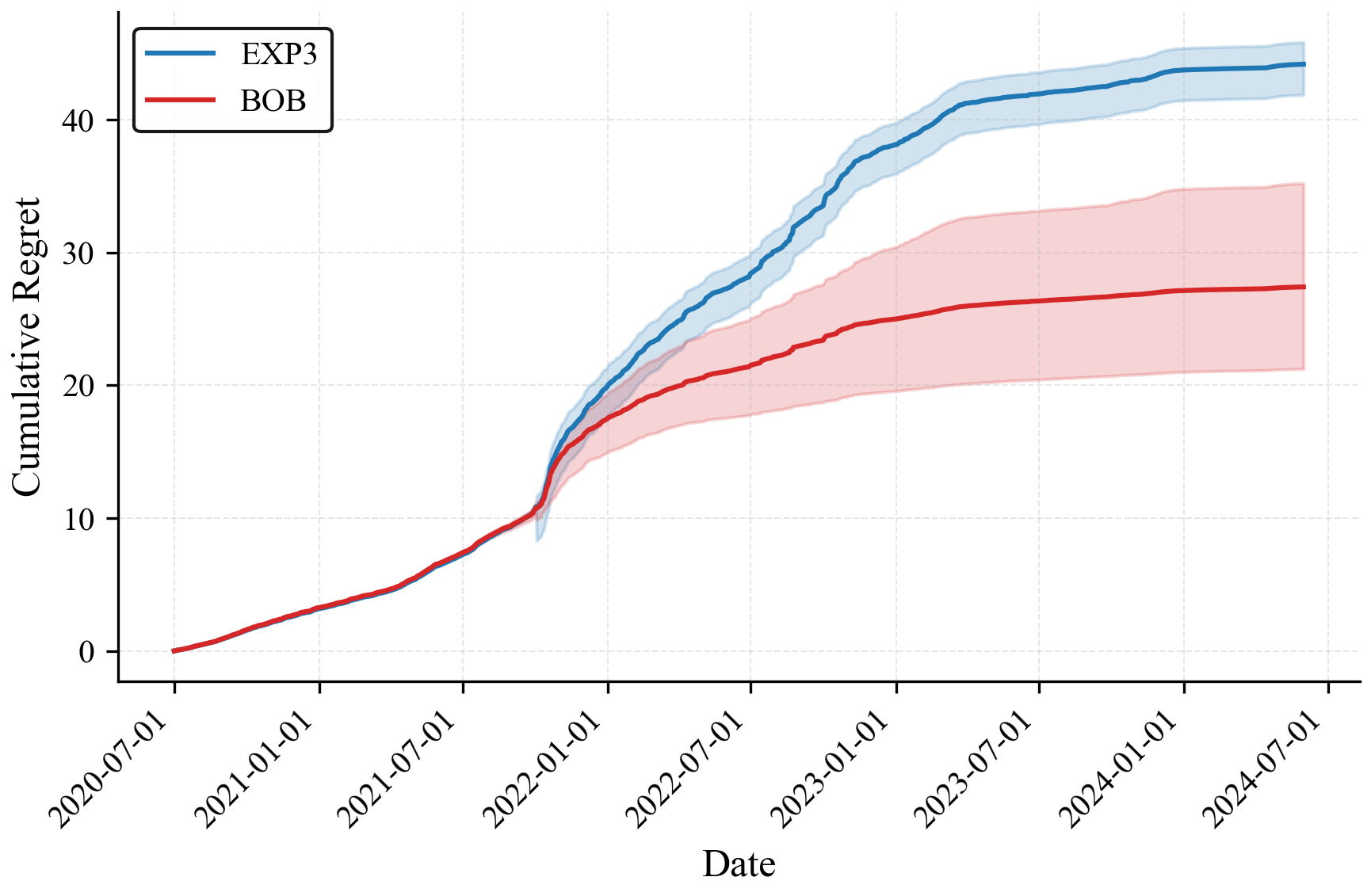}
        \caption*{(a) 00--04h}
    \end{subfigure}\hfill
    \begin{subfigure}{0.48\linewidth}
        \centering
        \includegraphics[width=\linewidth]{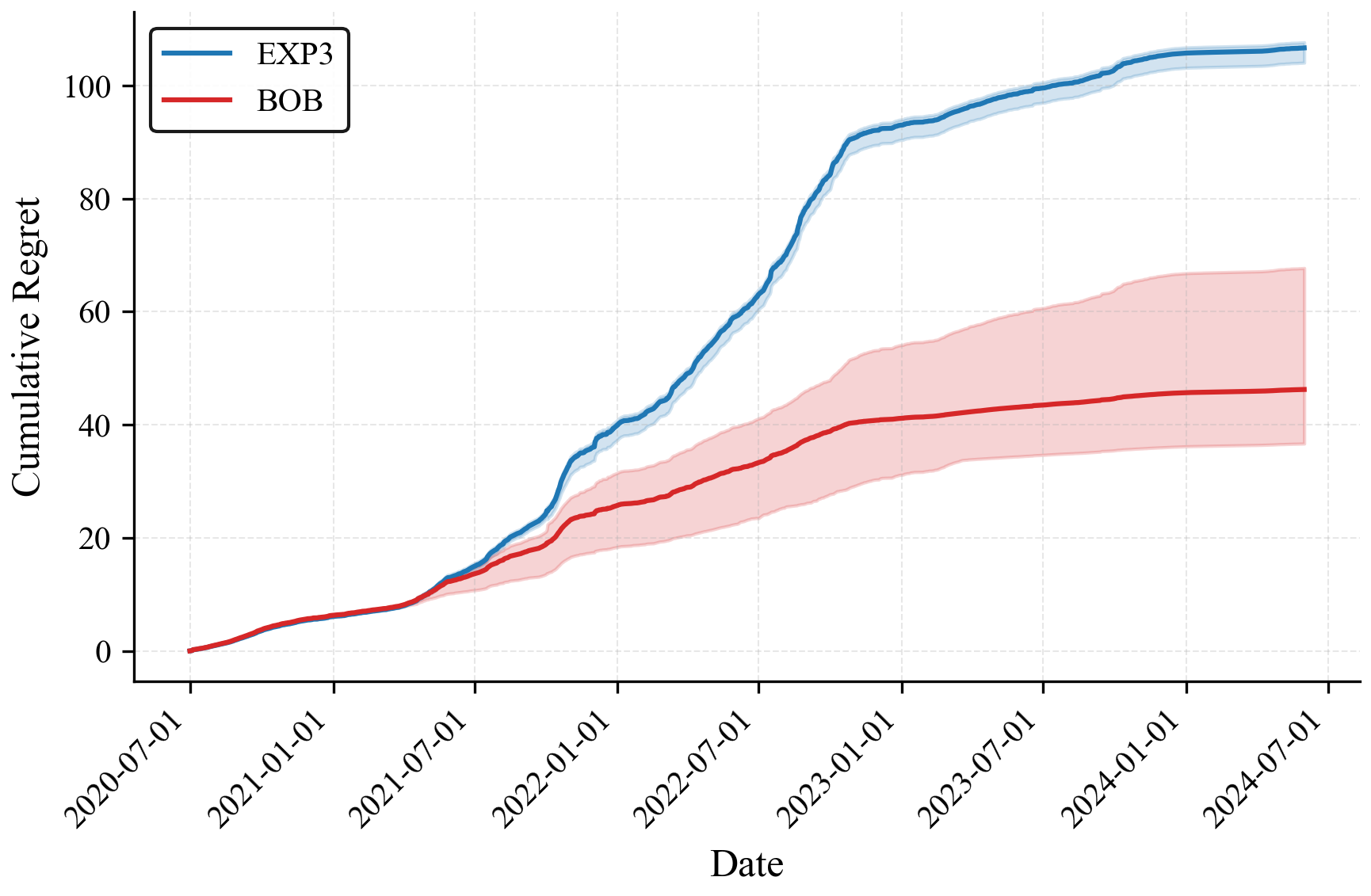}
        \caption*{(b) 16--20h}
    \end{subfigure}
    \caption{Cumulative regret on FCR for two products.}
    \label{fig:regret_00_04_12_16}
\end{figure}

\autoref{fig:regret_00_04_12_16} shows the regret incurred by the algorithms in our counterfactual experiments (backtesting). For this test, the learner is subjected to fully realistic conditions: it has no access to future data, and no data were used beforehand to train or calibrate it. Since the compared algorithms are stochastic, experiments were repeated to assess empirical averages and distributional behaviour of regret. The figure shows both the mean and the $2.5$--$97.5\%$ quantiles.

These experiments show that BOB incurs lower regret on the illustrated products, outperforming the EXP3 baseline in those settings.

\subsection{Non-stationarity}                                                                                    

\begin{figure}[!htbp]
    \begin{subfigure}{0.48\linewidth}
        \includegraphics[width=\linewidth]{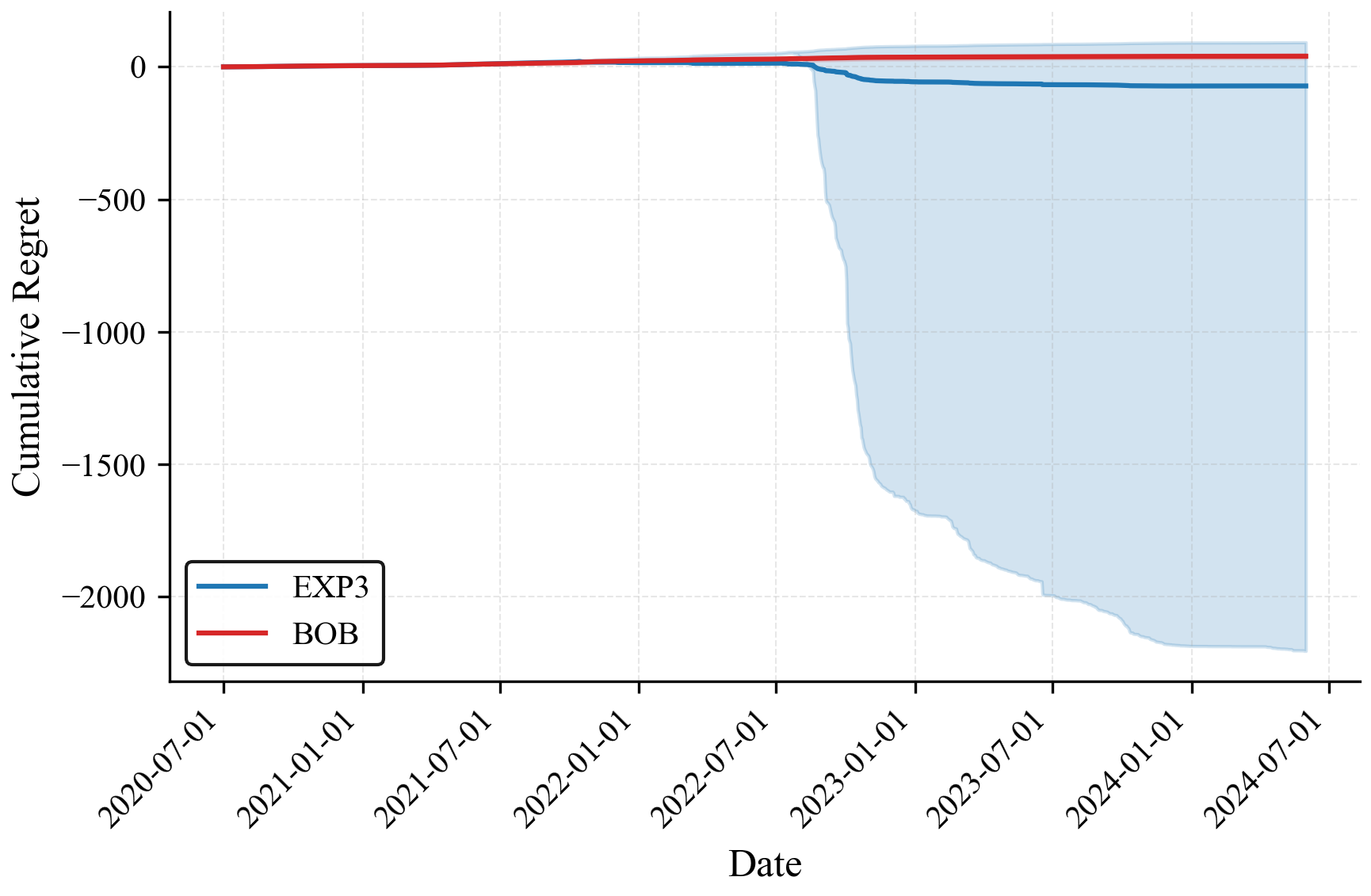}
        \caption{Product 8h-12h}
        \label{fig:regret_other}
    \end{subfigure} \hfill
    \begin{subfigure}{0.48\linewidth}
        \includegraphics[width=\linewidth]{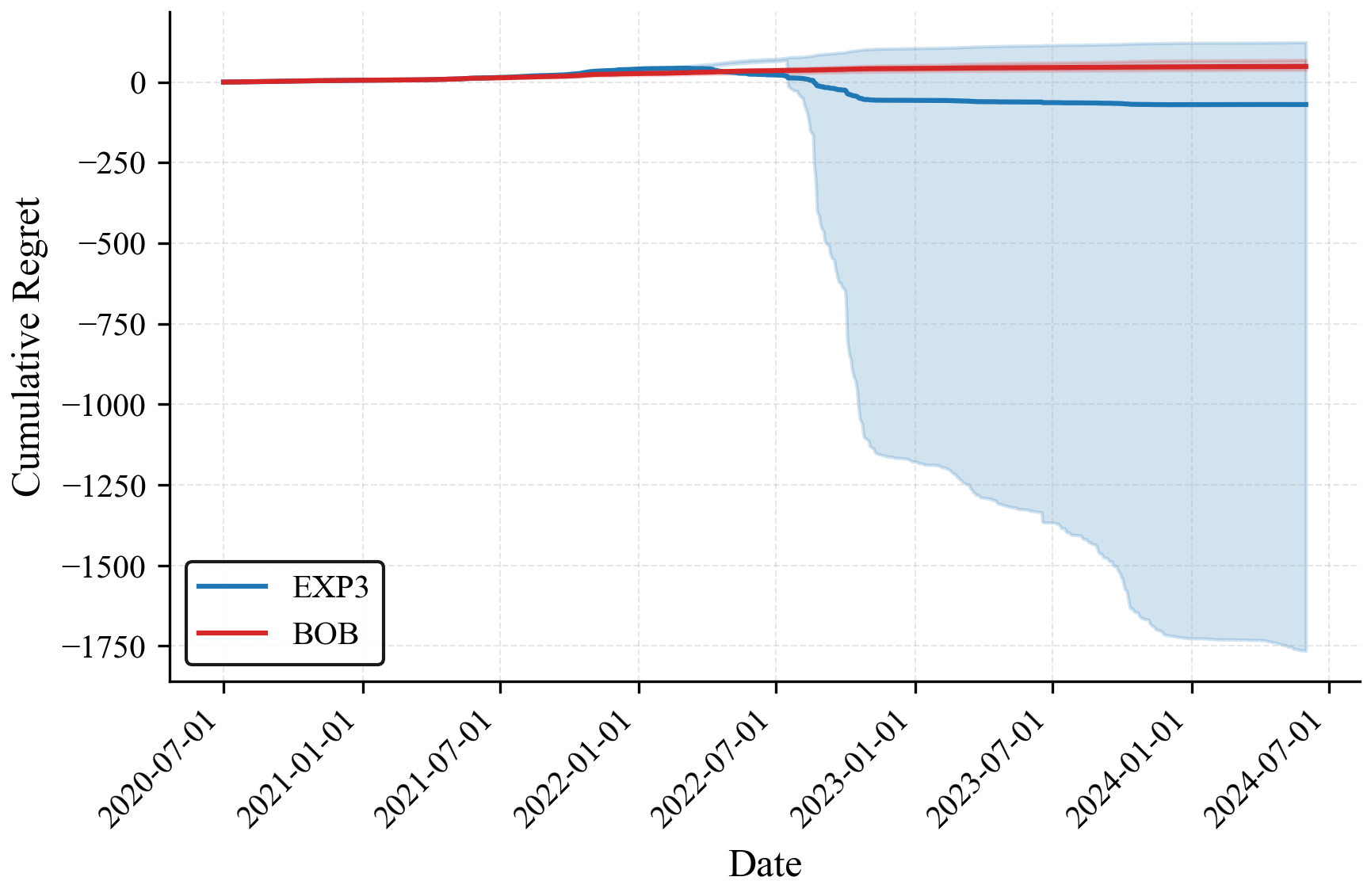}
        \caption{Product 12h-16h}
        \label{fig:regret_day_text}
    \end{subfigure}

    \caption{Day products regret comparison EXP3 and BOB (\autoref{algo : bob_bidder})}
\end{figure}

As shown in \autoref{fig:regret_day_text}, these positive results need to be tempered: we obtain drastically different results for the ``day'' products (08h--12h and 12h--16h). In these cases, while BOB behaves similarly, the regret incurred by the EXP3-based algorithms varies widely, with a significant probability of obtaining negative regret.

This behaviour is straightforward to interpret: it is consistent with stronger non-stationarity in opposing bids. For regret to become negative, submitted bids must, on average, outperform the best fixed bid in hindsight; this may happen when regime shifts change the best response over time. Since BOB is designed to exploit stochastic stability, it may fail to capture these shifts as effectively as more robust EXP3-type algorithms.
\FloatBarrier

\section{Conclusion}\label{conclusion}

This paper addresses the following question: how should a flexibility provider price its offers in the FCR market when competitors' bids are unknown? We first provide a detailed model of the market-clearing rules and then focus on the point of view of a single participant. Under mild assumptions, this reduces the bidding problem to a uniform-price auction, a simpler setting with a well-understood online learning structure.

Building on this reduction, we formulate bidding against unknown opposing bids as an online learning problem and adapt Best-of-Both-Worlds algorithms to the FCR setting. The resulting procedure uses only market-level bandit feedback (allocation and clearing price), relies on induced semi-bandit observations in the reformulated action space, and yields regret guarantees of $\mathcal{O}(\log T)$ in stochastic environments and $\mathcal{O}(\sqrt{T})$ in adversarial environments. The optimization step admits a tractable conditional-gradient implementation through an efficient linear minimization oracle on $\mathcal{H}$. In both synthetic experiments and backtests on real FCR data, these guarantees translate into competitive empirical performance.

From an operational perspective, the evidence supports a product-dependent policy. BOB is a strong default on products with more stable bid distributions, which in our data typically correspond to night/base-load-like periods. For products with stronger non-stationarity, notably day products, EXP3-type methods can be safer and may even outperform BOB by adapting better to regime shifts. In practice, this motivates an adaptive deployment strategy: select the learner by product and revise that choice over time as market stability changes.

Several extensions would broaden applicability. A natural next step is to handle participants active in multiple countries and account for interconnection constraints directly in the learning loop. Another is to include contextual covariates (e.g., plant availability, weather, or seasonality) to move beyond stationary product-level models. It would also be valuable to study strategic interactions when several participants learn simultaneously, and to test whether the same reduction-and-learning approach extends to other reserve or energy market mechanisms.

\section*{Acknowledgement}
Marius Potfer acknowledges the support of ANR through the PEPR IA FOUNDRY project (ANR-
23-PEIA-0003) and the Doom project (ANR-23-CE23-0002), as well as the ERC through the Ocean
project (ERC-2022-SYG-OCEAN-101071601).
Pierre Gruet and Cheng Wan acknowledge support from the FiME Lab.

\newpage
\bibliographystyle{cas-model2-names}

\bibliography{biblio}

\newpage
\appendix

\section{Appendix}\label{app}

\subsection{Some characteristic lemma of the FCR clearing}

The following lemma formalizes the (local) stability property of the FCR clearing that we use in the proof sketch of the market/auction equivalence: changing only the \emph{last accepted bid} of bidder $j$ (making it lower) or only the \emph{first rejected bid} (making it higher) cannot change bidder $j$'s allocation.
\begin{lemma}[Monotonicity of the clearing w.r.t. $j$'s bids]
\label{lem:monotonicity_clearing_bids}
Fix the bids of all participants other than $j$ (denoted $\bvec_{-j}$), and let $\bvec_j$ be a valid bid vector for bidder $j$. Let $\bigl((x_i)_{i\in\mathcal{J}}, (p_s)_{s\in\mathcal{S}}\bigr)$ be the clearing outcome for $(\bvec_j,\bvec_{-j})$, and write $x^\star := x_j$.

Assume that $\tilde{\bvec}_j$ is an alternative bid vector obtained from $\bvec_j$ such that
\[
\tilde{b}_{j,k} \le b_{j,k} \;\; \text{for all } k \le x^\star,
\qquad
\tilde{b}_{j,k} \ge b_{j,k} \;\; \text{for all } k > x^\star.
\]
Then the clearing outcome for $(\tilde{\bvec}_j,\bvec_{-j})$ satisfies $\tilde{x}_j = x^\star$ (i.e., bidder $j$'s allocation is unchanged).
\end{lemma}
\begin{proof}
We fix a deterministic tie-breaking rule in favor of bidder $j$ for the proof of this lemma\footnote{The same argument is unchanged for any deterministic tie-breaking rule fixed ex ante and independent of the local bid perturbation applied to bidder $j$.}.

The FCR clearing mechanism determines the allocation by minimizing the total procurement cost across the network. Let $C(y, \bvec_j)$ denote the minimal total procurement cost defined by the objective function \eqref{eq:fcr_minimization} under the constraint that participant $j$ is exogenously forced to provide exactly $y$ units. Because $x^\star$ is the optimal allocation under the original bids $\bvec_j$, it necessarily holds that
\[
C(y, \bvec_j) \ge C(x^\star, \bvec_j)
\]
for any alternative allocation $y$.

The key admissibility observation is the following. Let $\bigl((x_i^\star)_{i\in\mathcal{J}}, (p_s^\star)_{s\in\mathcal{S}}, \mathbf t^\star\bigr)$ be one clearing solution for $(\bvec_j,\bvec_{-j})$ with $x_j^\star=x^\star$. Under the considered perturbation, we only lower bids of bidder $j$ on accepted indices ($k\le x^\star$) and only increase bids on rejected indices ($k>x^\star$). Hence, at the original price vector $\mathbf p^\star$, bidder $j$'s accepted/rejected partition is preserved, so the quantities entering admissibility constraints remain valid for the same triple $\bigl((x_i^\star), (p_s^\star), \mathbf t^\star\bigr)$. In particular, the original clearing solution remains admissible for $(\tilde{\bvec}_j,\bvec_{-j})$.

As a consequence, the comparison step
\[
C(x^\star, \tilde{\bvec}_j) \le C(x^\star, \bvec_j)
\]
is justified by feasibility preservation (same candidate feasible point, same objective form), and the remaining argument compares alternative forced allocations $y$ against this preserved reference.

We evaluate the procurement cost under the modified bids $\tilde{\bvec}_j$.

For any candidate allocation $y < x^\star$, participant $j$ provides fewer units than in the optimal scenario. The replacement units required to satisfy the network balance \eqref{eq:clearing_allocation_balance} must be procured from competing participants, and these marginal replacement units set the local price $p_s$ according to \eqref{eq:clearing_prices_international}. Consequently, the bids of $j$ for $k \le y$ are strictly inframarginal. Therefore, lowering these bids does not alter total procurement cost:
\[
C(y, \tilde{\bvec}_j) = C(y, \bvec_j).
\]

For the optimal allocation $x^\star$, decreasing the bids of participant $j$ for units $k \le x^\star$ can only decrease or maintain the clearing price $p_s$. Therefore,
\[
C(x^\star, \tilde{\bvec}_j) \le C(x^\star, \bvec_j).
\]
Combining these observations yields, for any $y < x^\star$,
\[
C(y, \tilde{\bvec}_j) = C(y, \bvec_j) \ge C(x^\star, \bvec_j) \ge C(x^\star, \tilde{\bvec}_j).
\]

Symmetrically, for any candidate allocation $y > x^\star$, the system must purchase additional units from $j$ (indexed by $x^\star < k \le y$). Because $\tilde{b}_{j,k} \ge b_{j,k}$ for these units, the cost to procure $y$ units under $\tilde{\bvec}_j$ is at least as high as under $\bvec_j$. Thus,
\[
C(y, \tilde{\bvec}_j) \ge C(y, \bvec_j) \ge C(x^\star, \bvec_j) \ge C(x^\star, \tilde{\bvec}_j).
\]
Hence $C(x^\star, \tilde{\bvec}_j) \le C(y, \tilde{\bvec}_j)$ for all possible allocations $y$, proving that $x^\star$ remains optimal under $\tilde{\bvec}_j$. Therefore, $\tilde{x}_j = x^\star$.
\end{proof}

\subsection{Equivalence with uniform price auction} \label{app : lemma equivalence}

To fully prove the equivalence result, we first restate how the uniform price auction mechanism is defined: 

\paragraph{Uniform Price auction}  \label{app:paragraph:reserve-market} We describe below in detail the uniform price auction mentioned above, where $D_{max,s}$ units are to be bought. We describe this auction from the point of view of the flexibility provider, whose cost of producing the $\text{k}^\text{th}$-unit of electricity is denoted by $c_{k}\in [0,1]$, against opposing bids $\betavec$. The auction proceeds as follows:

\begin{enumerate}
\item The principal submits its bids $\bvec_j :=(b_{j,k})_{k\in[K]} \in B$, where $K=D_{max,s}$ and $B=\{(b_l)_{l \in [K]} \in \mathcal{B}^{D_{max,s}}, \text{ such that } 0 \leq b_{1}\leq b_{2} \leq \hdots \leq b_{K} \leq1 \}$.
\item The opposing bids $\betavec = (\beta_k)_{k \in [K]} \in B$ (in non-decreasing order)\footnote{We take $\betavec \in B$, this is without loss of generality : if $\boldsymbol{\beta}$ has more than $D_{max,s}$ elements then, when only keeping the $D_{max,s}$ the auction results remain the same regardless of $\bvec_j$  } are submitted to the auction.
\item The per-unit price is set as the $D_{max,s}^\text{th}$ lowest bid from $(\bvec_j,\betavec)$, denoted by $p^{\mathrm U} \left(\bvec_j, \betavec \right)$.
\item The principal sells to the auctioneer every unit of flexibility they proposed below the price $p^{\mathrm U} \left( \bvec_j,\betavec \right)$ and receives this price in exchange. Their allocation $x^{\mathrm U}\in[K]$ of production is therefore as follows\footnote{This allocation breaks ties in favor of the flexibility provider for simplicity; our results still follow through if tie-breaking is random or always against the principal.}:
\begin{align*}
x^{\mathrm U}\left (\bvec_j,\betavec \right ) & := 
\card{ \left \{k \in [K] \text{ s.t. } b_{j,k} \leq p^{\mathrm U} \left ( \bvec_j,\betavec \right ) \right \} } .\numberthis \label{app:def:allocation-elec-market}
\end{align*}
\end{enumerate}

Before providing the complete proof, to ensure clarity, we restate the theorem.
\mymaintheorem*

\begin{proof}[Formal proof of Theorem~\ref{lemma : equivalence market-auction}]
Recall that we focus on a flexibility provider $j\in\mathcal J_s$, located in country $s\in\mathcal S$, that we called the principal, we define \[
K:=D_{\max,s}.
\]
All bids of players different from $j$ are fixed and denoted by $\bvec_{-j}$.
For any admissible bid vector $\bvec_j=(b_{j,1},\dots,b_{j,K})\in B$, let
\[
x_j(\bvec_j):=x_j(\bvec_j,\bvec_{-j}),\qquad p_s(\bvec_j):=p_s(\bvec_j,\bvec_{-j})
\]
be bidder $j$'s FCR allocation and local price.

Our proof leverages the regularity of allocation shown by \autoref{lem:monotonicity_clearing_bids}, this allows us to fully characterize the allocation function by only characterizing allocation for simple "test profiles", detailed below. 

For each $m\in[K]$ and each $\alpha\in[0,1]$, define the test profile
\[
\psi^{m}(\alpha):=(\underbrace{0,\dots,0}_{m-1\text{ entries}},\alpha,\underbrace{1,\dots,1}_{K-m\text{ entries}})\in B,
\]
and
\[
g_m(\alpha):=\indicator\!\left\{x_j\bigl(\psi^{m}(\alpha)\bigr)\ge m\right\}.
\]
By Lemma~\ref{lem:monotonicity_clearing_bids}, each $g_m$ is a threshold function. Define
\[
\beta_{K-m+1}:=\sup\{\alpha\in[0,1]: g_m(\alpha)=1\}\in[0,1],\qquad m\in[K],
\]
and set $\betavec:=(\beta_1,\dots,\beta_K)$.

\textbf{Step 1 (allocation equivalence).}
Fix any $\bvec_j\in B$ and any $m\in[K]$.
By successive applications of Lemma~\ref{lem:monotonicity_clearing_bids},
\[
\indicator\{x_j(\bvec_j)\ge m\}=\indicator\{b_{j,m}\le\beta_{K-m+1}\}.
\]
Summing over $m$ gives
\begin{equation}
\label{eq:allocation_threshold_representation_appendix}
x_j(\bvec_j)=\sum_{m=1}^K \indicator\{b_{j,m}\le\beta_{K-m+1}\}.
\end{equation}
Equation \eqref{eq:allocation_threshold_representation_appendix} is exactly the uniform-price acceptance rule, hence
\[
x_j(\bvec_j)=x^{\mathrm U}(\bvec_j,\betavec).
\]

\textbf{Step 2 : price equivalence}
Let $k:=x_j(\bvec_j)$. By allocation admissibility in \eqref{eq : encadrement allox},
\begin{equation}
\label{eq:price_bracket_appendix}
b_{j,k}\le p_s(\bvec_j) < b_{j,k+1},
\end{equation}
with conventions $b_{j,0}:=0$ and $b_{j,K+1}:=1$.

From Step~1, since exactly the first $k$ units are sold,
\[
b_{j,k}\le\beta_{K-k+1},\qquad b_{j,k+1}>\beta_{K-k}.
\]
Define $q_k:=\max\{b_{j,k},\beta_{K-k}\}$, we show below that $p_s(\bvec_j)=q_k$.

Let us first suppose that $p_s(\bvec_j)>q_k$ and show it implies a contradiction. 
If $p_s(\bvec_j)>q_k$, consider the modified bids that only change $b_{j,k+1}$ and fix it to a value in $(q_k,p_s(\bvec_j))$.
By \eqref{eq:allocation_threshold_representation_appendix}, allocation remains $k$. Yet, by the constraint \eqref{eq:clearing_allocation_balance}, as well as the definition of $\bar{x}_j$ and $\underline{x}_j$, the allocation must be $k+1$, which is a contradiction. Hence
\[
p_s(\bvec_j)\le q_k.
\]

Let us now suppose that $p_s(\bvec_j)<q_k$ and show it implies a contradiction.
If $p_s(\bvec_j)<q_k$, either $p_s(\bvec_j)<b_{j,k}$ (contradicting \eqref{eq:price_bracket_appendix}) or
$p_s(\bvec_j)<\beta_{K-k}$. Suppose $p_s(\bvec_j)<\beta_{K-k}$, and modify only $b_{j,k+1}$ to a value in $(p_s(\bvec_j),\beta_{K-k}]$. By \eqref{eq:allocation_threshold_representation_appendix}, this local change moves the $(k+1)$-st unit to the accepted side of the threshold description, so the induced allocation is at least $k+1$. But by \eqref{eq:price_bracket_appendix}, the boundary at price $p_s(\bvec_j)$ corresponds to exactly $k$ accepted units for the principal. This is a contradiction. Therefore
\[
p_s(\bvec_j)\ge q_k.
\]
Thus,
\begin{equation}
\label{eq:price_max_identity_appendix}
p_s(\bvec_j)=\max\{b_{j,k},\beta_{K-k}\}.
\end{equation}

Since $k=x^{\mathrm U}(\bvec_j,\betavec)$, the $K$-th order statistic of $\bvec_j\cup\betavec$ is
\[
p^{\mathrm U}(\bvec_j,\betavec)=\max\{b_{j,k},\beta_{K-k}\}.
\]
Using \eqref{eq:price_max_identity_appendix}, we obtain
\[
p_s(\bvec_j)=p^{\mathrm U}(\bvec_j,\betavec).
\]

Therefore, for every $\bvec_j$, both allocation and price coincide between the FCR mechanism and the corresponding uniform auction, proving the theorem.
\end{proof}

\subsection{Rewriting the utility} \label{app : rewriting utility}

The point of this section is to restate/make the link to the correct notation with Lemma 1 and Lemma 3 from \cite{potfer2024improved}. This shows that the utility can be rewritten as a dot product between two vectors: $\mathbf{h}$ and $\mathbf{W}$, each depending either on $\bvec$ or $\boldsymbol{\beta}$. 

We begin by recalling how the utility is traditionally written in a uniform-price auction. In this section, since we focus on uniform auctions, we use $x(.)$ instead of $x^{\mathrm U} (.)$ and $p(.)$ instead of $p^{\mathrm U} (.)$.

\[ u(\bvec,\boldsymbol{\beta}) = \sum_{l=1}^{x(\bvec,\boldsymbol{\beta})} \left [ p(\bvec,\boldsymbol{\beta}) -c_l \right ]\]

We then move to writing this in a similar fashion as in \cite{potfer2024improved}, by using indicator functions, denoted $\indicator$. Then, we can just recall their Lemma 1 for the existence and uniqueness of $\mathbf{h} (\bvec)$ and their Lemma 3 to write the following formula for $\mathbf{W}(\betavec)$.

Let $N:=\frac{2K}{\varepsilon}$ and $\mathbf{W}(\betavec):=(w_k(\betavec))_{k=1}^{N}$. For each coordinate $k\in[N]$, define
\[
m_k:=\left \lfloor \frac{k}{2\varepsilon} \right \rfloor,\qquad
q_k:=\left(\frac{k}{2\varepsilon}-m_k\right)\varepsilon.
\]
Then
\[
w_k(\betavec):=
\begin{cases}
\sum_{l=1}^{m_k}\left(\beta_{K-m_k}-c_l\right)
& \text{if } \floor{\frac{k}{\varepsilon}} \text{ is odd and } q_k=\beta_{K-m_k},\\
\sum_{l=1}^{m_k}\left(q_k-c_l\right)
& \text{if } \floor{\frac{k}{\varepsilon}} \text{ is even and } \beta_{K-m_k}<q_k<\beta_{K-m_k+1},\\
0 & \text{otherwise.}
\end{cases}
\]
For boundary indices, we use the convention $\beta_0:=0$ and $\beta_{K+1}:=1$.

As for $\mathbf{h} (\bvec)$, instead of indexing it directly as a 1D vector, the mapping is based on a flattened 2D matrix, defined by binary variables that track the structural properties of the bid sequence across items and price levels. Let $\mathcal{K} = \{1, \frac{3}{2}, 2, \dots, K - \frac{1}{2}, K\}$ be the set of item indices including half-steps.

We now define these indicators without ambiguity. Let the price grid be
\[
\mathcal{J}_\varepsilon := \{0,\varepsilon,2\varepsilon,\dots,1\}.
\]
For any unit index $k\in[K]$ and any grid value $q\in\mathcal{J}_\varepsilon$, define
\[
h_{k,q}(\bvec):=\indicator\{b_{k}=q\}.
\]
For any $k\in[K-1]$ and any grid value $q\in\mathcal{J}_\varepsilon\setminus\{1\}$, define the half-step indicators
\[
h_{k+\frac12,q}(\bvec):=\indicator\{b_{k}<q< q +\varepsilon \le b_{k+1}\}.
\]
This interval condition records the price grid transition crossed between consecutive bids.

Finally, to compute the linear utility
\[
u(\bvec,\betavec)=\mathbf{h}(\bvec)^\top\mathbf{W}(\betavec),
\]
the pseudo-bid vector $\mathbf{h}(\bvec)\in\{0,1\}^{N+ K -\frac{1}{\epsilon}}$ is obtained by flattening the family
\[
\{h_{k,q}(\bvec):k\in[K],q\in\mathcal{J}_\varepsilon\}\cup\{h_{k+\frac12,q}(\bvec):k\in[K-1],q\in\mathcal{J}_\varepsilon\setminus\{1\}\}
\]
in lexicographic order (first by $k\in\mathcal{K}$, then by decreasing $q$).

\begin{proposition}[Bandit-to-semi-bandit reconstruction on active coordinates]
\label{prop:bandit_to_semibandit_reconstruction}
Fix a round $t$. Suppose the learner observes bandit feedback $(p_t,x^t)$ and knows its own submitted bid vector $\bvec^t$ and cost vector $(c_l)_{l\in[K]}$. Then for every played (active) coordinate $i$ with $h_i^t=1$, the quantity
\[
y_i^t:=h_i^tW_i(\betavec^t)
\]
is a deterministic function of $(\bvec^t,p_t,x^t)$ and therefore can be reconstructed from market-level bandit feedback.
\end{proposition}

\begin{proof}

Fix $t$ and denote by $u_t$ the realized utility:
\begin{equation}\label{app : utility equation from alloc and price}
u_t=\sum_{l=1}^{x_j^t}(p_t-c_l).
\end{equation}
We use Lemma 3 in \cite{potfer2024improved}: utility is decomposed into sub-utilities indexed by coordinates, and each sub-utility is an indicator of an auction outcome event (allocation level + price cell) times the corresponding payoff term. In particular, the events are disjoint, so at most one coordinate contributes a strictly positive sub-utility.

In our notation this gives

\[u_t=\sum_{i=1}^{N} h_i^tW_i(\betavec^t)
\]
And one notices from Lemma 3 in \cite{potfer2024improved} that each  $h_i^tW_i(\betavec^t)$ is of the form  \[
\indicator \{\{x^t = m_i \}\cap\{ p_t=q_i\}\} w_i,
\] The key point is that, once $\bvec^t$ is fixed, each $h_i(\bvec^t)W_i(\betavec^t)$ can be computed (because only one can be non-zero and total utility can be computed using \eqref{app : utility equation from alloc and price}).
Therefore, knowing both the allocation and realized price is enough to compute all $y_i^t$.

\end{proof}

\printcredits


\end{document}